\documentclass[10pt]{article}
\usepackage[a4paper, top=2.5cm, bottom=2.5cm, left=2.5cm, right=2.5cm]{geometry}
\usepackage{amssymb, amsmath}
\usepackage{graphicx}
\usepackage{tabularx}
\usepackage{booktabs}
\usepackage{array}
\usepackage{float}
\usepackage{multirow}
\usepackage{url}
\usepackage{caption}
\usepackage{subcaption}
\usepackage{pdflscape}
\usepackage{comment}
\usepackage{makecell}
\usepackage{mathrsfs}
\usepackage{algorithm}
\usepackage{algpseudocode}
\usepackage[numbers]{natbib}
\usepackage{soul}
\usepackage{xcolor}
\usepackage{hyperref}
\hypersetup{hidelinks}
\usepackage{amsthm}
\usepackage{fancyhdr}
\newtheorem{theorem}{Theorem}[section]

\newtheorem{proposition}[theorem]{Proposition}

\theoremstyle{definition}

\def\argmin{\mathop{\rm argmin}\limits}

\newcommand{\tb}{\textbf}

\title{LP-NAS: Linear Programming-based Neural Architecture Search}

\author{Abhishek Shukla$^{1}$ \and Ankur Sinha$^{2}$ \and Faiz Hamid$^{1}$}

\date{
$^{1}$Department of Management Sciences, IIT Kanpur, India\\
\texttt{abhiskl@iitk.ac.in, fhamid@iitk.ac.in}\\[6pt]
$^{2}$Krishnamurthy Tandon School of AI, IIM Ahmedabad, India\\
\texttt{asinha@iima.ac.in}
}

\begin{document}

\maketitle
\begin{abstract}
Neural Architecture Search (NAS) aims to automate neural network architecture design, reducing reliance on human expertise. Among the various NAS methods, differentiable NAS has gained prominence due to its efficiency and accuracy compared to conventional NAS approaches. Since differentiable NAS relaxes the architecture search space into a continuous domain, it is possible to apply principles from continuous optimization to NAS. In this paper, we propose Linear Programming-based NAS (LP-NAS), a mathematical programming-based framework for differentiable NAS that is applicable to a wide range of continuous search spaces. LP-NAS formulates a linear program (LP) using the validation-loss gradient and the training-loss Hessian to compute an architecture update direction that improves generalization while preserving the optimality of the model parameters. By following this LP-derived descent direction, LP-NAS efficiently navigates the architecture search space, leading to faster and more effective architecture optimization. We introduce two computationally efficient variants of LP-NAS, namely S-LP-NAS and R-LP-NAS. Applying LP-NAS to the Differentiable Architecture Search (DARTS) search space results in two algorithmic variants, S-LP-DARTS and R-LP-DARTS. Both variants achieve faster convergence and significantly higher validation performance during the early search iterations than the standard DARTS algorithm. Extensive experiments on CIFAR-10 and CIFAR-100 show that LP-DARTS outperforms standard DARTS in both the architecture search and evaluation phases. Additionally, we compare our approach with several DARTS variants (P-DARTS, PC-DARTS, and STO-DARTS) on the CIFAR-10 dataset and demonstrate its effectiveness. Furthermore, we validate the transferability of the discovered architectures through experiments on the ImageNet dataset.
\end{abstract}


\section{Introduction}
Artificial Neural Networks (ANNs) are machine learning models designed to process data in a manner inspired by the human brain. ANNs are characterized by architecture parameters (hyperparameters) which are not learned from the training data but are instead set prior to training. Hyperparameter Optimization (HPO) is a technique aimed at determining the optimal values of the hyperparameters for machine learning models. Traditionally, expert knowledge was essential in designing neural network architectures. However, with advances in artificial intelligence, engineers and scientists are increasingly focused  on automating the neural network design process. This emerging area of research, referred to as Neural Architecture Search (NAS), a sub-field of HPO and Automated Machine Learning (AutoML), aims to facilitate the automated development of neural network architectures for specific requirements.

Mathematically, NAS and other HPO problems, as formalized in \cite{elsken2019neural,ren2021comprehensive}, are bilevel optimization problems \cite{sinha2017review, pujara2025review}, in which an outer (leader) problem selects the architecture and an inner (follower) problem trains the corresponding weights, as shown in Formulation~\eqref{formulation_1}. These problems are typically solved under the optimistic bilevel assumption. In this formulation, \(\mathcal{L}_v\) and \(\mathcal{L}_t\) represent the validation loss and the training loss functions. The lower-level variable ($W$) corresponds to the neural network's weights and biases, while the upper-level variable ($A$) denotes the architecture parameters. When considering NAS as a bilevel optimization problem, it is noteworthy that the dimensionality of $W$ significantly exceeds that of $A$.
\begin{equation}\label{formulation_1}
	\begin{aligned}
		\min_A \quad & \mathcal{L}_v(A, \hat{W}) \\
		\text{st.} \quad & \hat{W} \in \arg\min_{W \in \mathscr{W}} \mathcal{L}_t(A, W) \\
		& A \in \mathscr{A}
	\end{aligned}
\end{equation}

NAS gained significant attention following the publication of the research paper \cite{zoph2017neural}. Since then, this domain has experienced substantial growth, evidenced by the proliferation of related publications in recent years \cite{white2023neural}. Architectures designed through NAS have demonstrated exceptional performance across diverse application areas, such as medical imaging \cite{xie2024lightweight, weng2019unet, wang2024mednas}, language modeling \cite{sarah2024llama}, translation \cite{chitty2022neural}, smart city applications \cite{li2020autost}, image classification and object detection \cite{zoph2017automl}, and semantic segmentation \cite{liu2019auto}. In many instances, NAS-designed architectures have outperformed those crafted by human experts. Advancements in automated neural network design have democratized the architecture design process, reducing the need for specialized knowledge. Current progress in NAS demonstrates its potential to revolutionize various application domains and research fields \cite{meng2024evolution}.

Due to the combinatorial nature and NP-hard complexity of NAS, solving this problem is computationally challenging and time-intensive. Specifically, Reinforcement Learning (RL) \cite{zoph2017neural} and Evolutionary Computation (EC) \cite{liu2021survey}-based algorithms require approximately three orders of magnitude more GPU days than gradient-based methods \cite{DBLP:conf/iclr/LiuSY19} due to the exploration of discrete architectural spaces and treating the objective function as a black box rather than a functional form. In contrast, gradient-based NAS approaches lack such intricacies. 

Since NAS emerged, various benchmarks have been developed to assess and compare new NAS methods. Differentiable Architecture Search (DARTS) has emerged as one of the most widely adopted benchmarks in the NAS field. The DARTS algorithm pioneered the use of gradient-based optimization techniques within NAS. Following its introduction, several variants of DARTS have been developed to improve particular aspects of the original method. Some of them are given in Appendix~\ref{app1}. The continuous advancements and broad usage of DARTS highlight its importance and substantial influence within the NAS field. Although these methods have significantly improved DARTS in terms of search efficiency, robustness, and optimization stability, they primarily focus on modifying the search strategy, search space, or regularization techniques. Comparatively less attention has been devoted to exploiting the mathematical structure of the underlying bilevel optimization problem, despite DARTS being naturally formulated as one. One of the early works to model HPO as a bilevel optimization problem is~\cite{bennett2008bilevel}, which has inspired various studies~\cite{shukla2026bilevel,sinha2025linear,sinha2024gradient,okuno2021lp,franceschi2018bilevel}. Beyond NAS, more recently, the mathematical structure of bilevel optimization-based regularization problems has also been exploited for overfitting control in large language models, including transformer architectures, as demonstrated in \cite{shukla2026lift}. The current work builds on this line of research by exploiting the mathematical properties of continuous hyperparameter optimization through its bilevel formulation.
\begin{figure}[htbp]
\begin{center}
    \includegraphics[width=13.5cm, height=8.5cm]{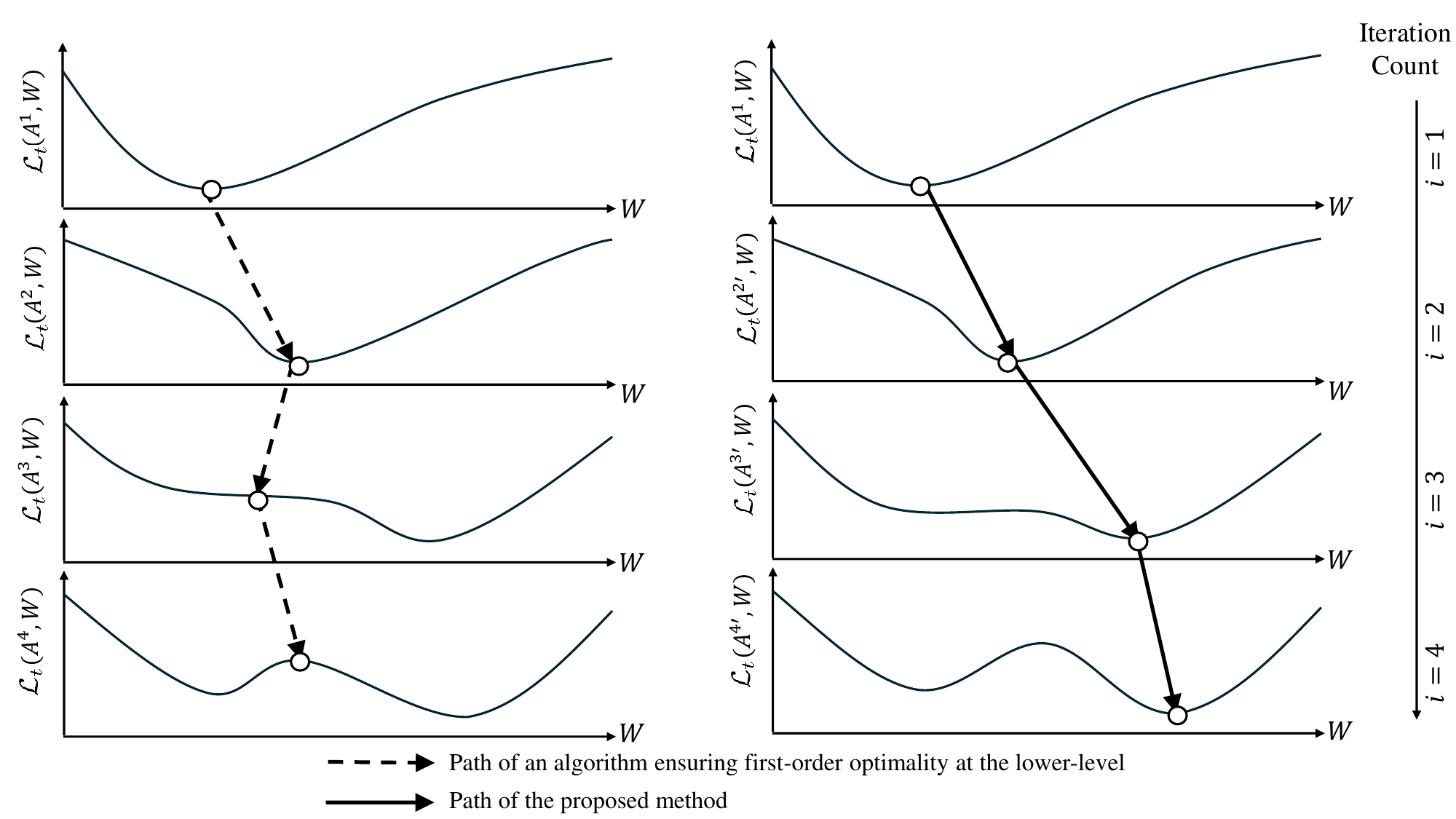}
    \caption{Optimization paths of a first-order method and our proposed approach.}
    \label{arch_search}
\end{center}
\end{figure}

We propose a novel NAS approach based on linear programming within the differentiable framework. Our method leverages the Hessian information from the lower-level optimization, together with the validation gradient from the upper-level, to formulate a Linear Program (LP) that computes a descent direction for the upper-level while ensuring that the architecture search is performed within the optimal valley of the lower-level problem. This approach, referred to as LP-NAS (an amalgamation of the LP-based bilevel descent method and differentiable NAS), distinguishes itself by maintaining model optimality throughout the architecture update process, unlike many existing NAS and HPO algorithms, which assume model optimality solely from the condition that the gradient of the training loss with respect to the model parameters is zero. While a zero gradient is a necessary condition, they are not sufficient for optimality. In contrast, our method ensures lower-level optimality while updating the architecture parameters and model parameters along the descent direction, providing a robust approach to optimization and effectively minimizing validation loss. Figure~\ref{arch_search} illustrates the issue of relying on the first-order lower-level optimality conditions alone and then optimizing the architecture parameters. To begin with, let the optimal model parameters corresponding to the initial architecture \( A^1 \) be known. Changing the architecture parameters (e.g., from \(A^2\) to \(A^3\)) across iterations, while ensuring first-order optimality at the lower-level, may result in updated weights corresponding to \(A^3\) that do not sufficiently satisfy the lower-level minimization condition. Consequently, although the upper-level objective (validation loss) may improve, the minimization of the lower-level training loss is not guaranteed. This imbalance can lead to overfitting on the validation loss while disregarding the sufficient minimization conditions for the training loss. A common approach to mitigate this issue is to re-solve the lower-level problem with a fixed \(A^3\), as done by many optimization methods. However, this approach can still pose challenges because the optimization begins from a state where the gradient is nearly zero. In such scenarios, the solution may be at a saddle point or even a local maximum of the lower-level problem while still achieving a relatively low validation loss. The final architecture (say $A^4$) may also not be optimal for the bilevel program. In contrast, the proposed method ensures that lower-level optimality is maintained while updating both architecture and model parameters along the descent direction, as illustrated in Figure~\ref{arch_search}. This allows for efficient convergence to model and architecture parameters that yield superior performance on validation/test datasets. 

In a nutshell, the key contributions of our work are:
\begin{enumerate}
    \item \tb{Linear pogramming-based hyperlocal search:} This approach leverages second-order information to compute architecture updates that improve validation performance while preserving model optimality.
    \item \tb{Efficient variants for faster convergence:} We propose two computationally efficient variants of LP-NAS, namely S-LP-NAS and R-LP-NAS, which differ in their model parameter selection strategies. Both variants accelerate convergence and improve validation performance during the early stages of the architecture search process.
    \item \tb{Superior performance on benchmark datasets:} Extensive experiments on the CIFAR-10 and CIFAR-100 datasets demonstrate that LP-DARTS, an instance of LP-NAS for the DARTS search space (DSS), consistently outperforms standard DARTS during both the architecture search and evaluation phases, resulting in superior architectures. 
    \item \tb{Comparison with DARTS variants and transferability:} We perform a comprehensive comparison with several DARTS variants (P-DARTS, PC-DARTS, and STO-DARTS) on the CIFAR-10 dataset, demonstrating the effectiveness of LP-DARTS, and validate the transferability of the discovered architectures through experiments on the ImageNet dataset.
\end{enumerate}

The paper is organized as follows. Section~\ref{sec:preliminaries} covers the foundational concepts and background necessary for the development of the LP-NAS algorithm. Section~\ref{sec:methodology} introduces the proposed LP-NAS approach in detail. In Section~\ref{sec:exp_results}, we present the experimental results for architecture optimization on the CIFAR-10 and CIFAR-100 datasets and compare the performance of LP-DARTS with that of the original DARTS algorithm. We also compare the architectures obtained in this work with those reported in the literature on the CIFAR-10 dataset. Furthermore, we evaluate the transferability of the architectures obtained on CIFAR-10 to the ImageNet dataset and discuss the insights gained from the experimental results. Finally, Section~\ref{sec:conclusions} offers concluding remarks, outlines the challenges and limitations encountered during the architecture search process, and suggests potential directions for future research.

\section{Preliminaries}
\label{sec:preliminaries}
In this section, we discuss the hypergradient-based approach for bilevel optimization and lay the foundation for developing the LP-NAS algorithm.

\subsection{Hypergradient}
The gradient of the upper-level objective (validation loss) with respect to the upper-level decision variable \(A\) is called the hypergradient. One of the most well-known optimization methods for NAS is based on an approximate hypergradient \cite{giovannelli2021inexact}. This hypergradient is used to update the architecture parameters to improve validation performance while accounting for the model weights obtained after one step of gradient descent on the training loss. The hypergradient is approximated as:
\[
\nabla_A \mathcal{L}_v(A, \hat{W}) \approx \nabla_A \mathcal{L}_v\left(A, W - \xi \nabla_W \mathcal{L}_t(A, W)\right),
\]
where $\xi$ is a small learning rate used for one-step unrolled optimization of the model training problem. Setting $\xi = 0$ yields a first-order approximation, while $\xi \neq 0$ introduces a second-order correction:
\begin{equation}\label{approx_hypergrad}
  \nabla_A \mathcal{L}_v(A,\hat{W}) \approx \nabla_A \mathcal{L}_v(A, W') - \xi \nabla^2_{A, W} \mathcal{L}_t(A, W) \nabla_{W'} \mathcal{L}_v(A, W'),  
\end{equation}
with $W' = W - \xi \nabla_W \mathcal{L}_t(A, W)$. The second term captures the interaction between the architecture parameters and the network weights through the mixed Hessian of the training loss, thereby accounting for their influence on the approximate hypergradient. This term can be efficiently approximated using finite-difference schemes. For further details on DARTS, see \cite{DBLP:conf/iclr/LiuSY19}.

\subsection{Hyperlocal Search using LP}
To address the bilevel optimization problem described in Formulation \eqref{formulation_1}, it is required to compute the descent direction/gradient of the upper-level objective. This is achieved by formulating and solving an LP, giving a descent direction for the upper-level objective that ensures the maintenance of lower-level optimality conditions when upper-level and lower-level variables are updated along this direction (Proposition~\ref{prop:hls_optimality}). This descent direction facilitates efficient exploration of the continuous hyperparameters and simultaneously updates model parameters. We refer to this search using LP as \textit{Hyperlocal Search (HLS)}. The LP for HLS is given as follows, which is a relaxed formulation of the original Second-Order Cone Program (SOCP) (see Appendix~\ref{app2} for the derivation), whose optimal solution gives the steepest descent direction.
\begin{equation}\label{devised_LP}
	\begin{aligned}
		\min_{d_A, d_W} \quad & \left\langle \begin{bmatrix}
    \nabla_A \mathcal{L}_v(A^0, W^0) \\
    \nabla_W \mathcal{L}_v(A^0, W^0)
\end{bmatrix}, \begin{bmatrix}
    d_A \\
    d_W
\end{bmatrix} \right\rangle \\
		\text{subject to:} \quad & \begin{bmatrix}
        H_{21} & H_{22}
    \end{bmatrix} 
    \begin{bmatrix}
        d_A \\
        d_W
    \end{bmatrix} = 0\\
		\quad & -1 \leq d_A \leq 1 \\
	\end{aligned}
\end{equation}

It is noteworthy that the set of linear constraints in the LP ensures that $[d_A, d_W]^T$ is orthogonal to $[H_{ij}]_{i=p+1, j=1}^{i=p+q, j=p+q}$.\footnote{ With some abuse of terminology, we refer to the rectangular matrix, $[H_{ij}]_{i=p+1, j=1}^{i=p+q, j=p+q}$, as the Hessian matrix. $\nabla_{(A, W)}^2 \mathcal{L}_t(A^0, W^0)=\begin{bmatrix}
    H_{ij}
\end{bmatrix}_{i=1, j=1}^{i=p+q, j=p+q} = \begin{bmatrix}
    [H_{ij}]_{i=1, j=1}^{i=p, j=p} & [H_{ij}]_{i=1, j=p+1}^{i=p, j=p+q} \\
    [H_{ij}]_{i=p+1, j=1}^{i=p+q, j=p} & [H_{ij}]_{i=p+1, j=p+1}^{i=p+q, j=p+q}
\end{bmatrix}=\begin{bmatrix}
    H_{11} & H_{12} \\
    H_{21} & H_{22}
\end{bmatrix}$} This means that any infinitesimal change in $(A^0, W^0)$ along this direction will continue to guarantee lower-level optimality.
Solving the LP \eqref{devised_LP} we obtain a descent direction ($d_A^*, d_W^*$), which provides models with different architectures, $\mathcal{M}(A^0+td_A^*, W^0+td_W^*)$, when $t$ is varied. With $t>0$, one moves along the direction that improves validation loss (optimal objective value for the linear programming problem is always $\leq 0$ as proved in the Proposition~\ref{prop:hls_improvement}).

\begin{proposition}\label{prop:hls_optimality}
(Preservation of Lower-Level Optimality after the HLS update).  
Let the model parameters \((W^0)\) be initially optimal at the lower-level for given architecture parameters \((A^0)\). If the HLS is performed using the proposed linear program, the model parameters continue to remain optimal after the HLS update. (proof in Appendix~\ref{app3}.)
\end{proposition}

\begin{proposition}\label{prop:hls_improvement}
(Validation Performance Improvement via HLS). The HLS results to an improvement in validation performance in most cases, and in the worst case, leaves it unchanged. (proof in Appendix~\ref{app4}.)
\end{proposition}

\section{Methodology}\label{sec:methodology}
To utilize the LP devised in \eqref{devised_LP}, the hyperparameters considered must be continuous. As previously explained, differentiable search spaces, such as that of DARTS, employ a continuous relaxation of the categorical operation choices on each edge by applying the SoftMax function over all candidate operations, thereby transforming the discrete search space into a continuous one. Therefore, the devised LP can be employed to conduct HLS, enabling efficient exploration of the architecture and model parameters. 

Once the descent direction ($d_A^*, d_W^*$) is obtained by solving the LP, we simultaneously update the hyperparameters and model parameters by a step size $(t)$ in the descent direction. We provide the details of a basic version of the proposed method, LP-NAS, in Algorithm~\ref{lpnasalgo}.

\begin{algorithm}
\caption{LP-NAS}
\label{lpnasalgo}
\begin{algorithmic}[1]
\setlength{\itemsep}{3pt}
\State Set iteration count $k=0$, initialize architecture parameters $A^0$ randomly and create the neural network using this architecture.
\State Lower-level optimization: Solve for $W^0$ for the given $A^0$, i.e. $W^0 = \argmin_{W \in \mathscr{W}} \mathcal{L}_t(A^0, W)$
\While{$k < k^\text{max}$}
    \State \parbox[t]{\dimexpr\linewidth-\algorithmicindent}{Hessian approximation: Use an efficient approximation approach to get the approximation of the Hessian matrix $[H_{ij}^k]_{i=p+1, j=1}^{i=p+q, j=p+q}$ at $(A^k, W^k)$.}
    \State \parbox[t]{\dimexpr\linewidth-\algorithmicindent}{Validation gradient: Calculate the validation gradient: $[\nabla_A \mathcal{L}_v(A^k, W^k), \nabla_W \mathcal{L}_v(A^k, W^k)]^T$.}
    \State \parbox[t]{\dimexpr\linewidth-\algorithmicindent}{LP formulation and solution: Formulate and solve the LP to get the descent direction ($d_{A^k}^{*}, d_{W^k}^{*}$).}
    \State \parbox[t]{\dimexpr\linewidth-\algorithmicindent}{Hyperlocal search: Update the architecture and the model parameters along the descent direction,}
    \begin{align*}
    (A^{k+1}, V^{k+1}) &\leftarrow (A^k, W^{k}) + t(d_{A^k}^{*}, d_{W^k}^{*})
    \end{align*}
    \State \parbox[t]{\dimexpr\linewidth-\algorithmicindent}{Lower-level optimization: Solve for $W^{k+1}$ for the given $A^{k+1}$ with a warm start from $V^{k+1}$.}
    \begin{align*}
    W^{k+1} = \argmin_{W \in \mathscr{W}} \mathcal{L}_t(A^{k+1}, W)
    \end{align*}
    \State \parbox[t]{\dimexpr\linewidth-\algorithmicindent}{Increment iteration count: $k \leftarrow k+1$}
\EndWhile
\State Get the final architecture from the learned architecture parameters.
\end{algorithmic}
\end{algorithm}

\subsection{Improving Memory and Computational Efficiency}
A direct implementation of the proposed LP formulation is computationally impractical for modern deep neural networks because the number of model parameters is typically several orders of magnitude larger than the number of architecture parameters, i.e., $p \ll q$. Consequently, computing and storing the Hessian with respect to all model parameters incurs prohibitive memory and computational costs.

To address this challenge, LP-NAS computes a \emph{reduced Hessian} by retaining all architecture parameters while considering only a small subset of the model parameters. Specifically, all $p$ architecture parameters and only $\kappa q$ selected model parameters ($\kappa \ll 1$) are used to construct the LP. Accordingly, the validation gradient is also computed only for these selected parameters, resulting in a reduced LP that is significantly more memory- and compute-efficient while preserving the essential second-order information required for optimization. The detailed formulation, computational complexity, and memory analysis of the reduced Hessian are provided in Appendix~\ref{app5}. The following section describes the strategies used to select the subset of model parameters employed in the reduced Hessian computation.

\subsection{Model Parameter Selection for Reduced Hessian Computation}

To compute the reduced Hessian matrix efficiently, LP-NAS computes the Hessian with respect to only a subset of the model parameters. We propose two parameter-selection strategies for selecting $\kappa q$ parameters from the total of $q$ model parameters.\\

\noindent \textbf{Specific parameter selection (S-LP-NAS).}
This strategy selects the model parameter tensor that is expected to have the greatest influence on the optimization process. The selection is based on the normalized magnitude of the training-loss gradient, referred to as the Gradient Norm Per Parameter (GNPP). The tensor with the largest GNPP value is selected, and all parameters within that tensor are used for reduced Hessian computation. Since this strategy deterministically selects a specific tensor, the resulting algorithm is referred to as \textbf{S-LP-NAS}, where \emph{S} denotes \emph{specific}.\\

\noindent \textbf{Random parameter selection (R-LP-NAS).}
As an alternative, we randomly select one or more model parameter tensors whose sizes satisfy a predefined upper bound on the reduced Hessian size. The reduced Hessian is then computed using all parameters belonging to the selected tensors. By varying the selected tensors across iterations, this strategy provides diverse update opportunities for different parts of the network. The corresponding algorithm is referred to as \textbf{R-LP-NAS}, where \emph{R} denotes \emph{random}.

The mathematical formulation of both parameter-selection strategies is provided in Appendix~\ref{app6}.

\subsection{Time Complexity}
In Algorithm~\ref{lpnasalgo}, let \( k^{\text{max}} \) denote the number of outer iterations in the bilevel optimization framework. Each outer iteration consists of three main components: (i) an L-BFGS-based Hessian approximation with memory size \( m \), (ii) an LP-based architecture update, and (iii) lower-level model optimization. Let \( p \) denote the number of architecture parameters and \( \kappa q \) the number of selected model parameters. Define the total number of optimization variables as \( n = p + \kappa q \). Further, let \( E_D \) denote the number of training samples processed per outer iteration with batch size \( B \), and let \( C_P \) represent the cost of a single forward–backward pass. The overall simplified complexity can be given as follows (see Appendix~\ref{app7} for details).
\begin{equation}
O\!\left(
k^{\text{max}} \left[
n^2 + \frac{E_D}{B} C_P
\right]
\right).
\end{equation}
This shows that the algorithm scales quadratically with the number of optimization variables and linearly with the effective data processed per iteration. All components are polynomial in problem size, and no exponential or combinatorial operations are involved, making the method computationally tractable for practical NAS settings. In practice, the dominant costs arise from the LP solve and the lower-level training.

\section{Experimental Results}\label{sec:exp_results}
In this section, we perform architecture searches using the two LP-DARTS approaches and compare their performance with DARTS. The loss and accuracy during the bilevel search are recorded for comparison purposes. The experiments utilized Convolutional Neural Network (CNN) for multi-class classification on the CIFAR-10 and CIFAR-100 datasets. Each dataset consists of 50,000 training and 10,000 test images, all in RGB format at \(32 \times 32\) pixels. CIFAR-10 contains 10 classes, while CIFAR-100 has 100 classes. All experiments were conducted on a high-memory HPC node featuring Intel Xeon Platinum 8268 processor (2.9 GHz, 48 cores) with 768 GB DDR4 RAM, and a 480 GB SSD. Despite the processor's multi-core capability, we restricted all executions to a single core for a fair comparison. After conducting the architecture search experiments, we evaluated the resulting architectures on their respective datasets. To evaluate the transferability and generalization capability of the searched architectures, architectures obtained from architecture search on CIFAR-10 were also evaluated on the ImageNet32 dataset. The ImageNet32 dataset used here is a downsampled version of the ILSVRC 2012 ImageNet classification dataset, comprising 1{,}281{,}167 training images and 50{,}000 validation images, uniformly resized to $32 \times 32$ pixels in RGB format. It includes 1{,}000 object classes, preserving the fine-grained classification challenge of the original ImageNet while maintaining computational tractability. The architecture search and evaluation results are detailed as follows.

\subsection{Architecture Search}
The experimental configuration followed the guidelines from the original DARTS study \cite{DBLP:conf/iclr/LiuSY19}. For all algorithms (DARTS, S-LP-DARTS and R-LP-DARTS), the search for optimal cells was conducted with a batch size of 96, splitting the dataset into equal-sized training and validation sets. The architecture consisted of 8 cells: 6 normal cells and 2 reduction cells. The two reduction cells were positioned at 1/3 and 2/3 distances deep in the network. All operations in the normal cells were performed with stride 1, while stride 2 was used in the reduction cells to downsample the feature maps by half. The operation set $\mathcal{O}$ considered in this work is the same as that used in DARTS article \cite{DBLP:conf/iclr/LiuSY19}.

We employed stochastic gradient descent (SGD) with momentum to optimize the model parameters. The initial learning rate was set to \(2.5 \times 10^{-2}\), momentum to 0.9, and weight decay to \(3.0 \times 10^{-4}\), with a cosine annealing schedule to gradually reduce the learning rate to \(2.5 \times 10^{-5}\). Architecture search using LP-DARTS used a constant step size of \(2.75 \times 10^{-3}\) for HLS across all epochs. We performed two architecture updates per epoch in all our architecture search experiments while optimizing architecture using LP-DARTS. This required the computation of the reduced Hessian matrix, reduced validation gradient, and the subsequent formulation and solution of the reduced LP \eqref{devised_LP} for each update. We utilized three random mini-batches from the training and validation sets to approximate the reduced Hessian and compute the reduced validation gradient. In R-LP-DARTS, the size of each selected model parameter tensor was limited to 1025 \((\varkappa_{CIFAR-10} \approx 0.053 \%, \varkappa_{CIFAR-100} \approx 0.052 \%)\) and we used only one model parameter tensor for HLS. To solve the LP formulations during optimization, we used the IBM CPLEX solver via the DOcplex Python API. Since CPLEX predominantly requires CPU resources, S-LP-DARTS and R-LP-DARTS necessitated the execution on a CPU. To ensure a fair comparison, DARTS was also run on the CPU.
\begin{table}[htpb]
\centering
\caption{Performance metrics of NAS methods on CIFAR-10 and CIFAR-100.}
\label{tab:combined_nas_search}
\footnotesize
\resizebox{0.85\textwidth}{!}{
\begin{tabular}{lc@{\hspace{0.4cm}}c@{\hspace{0.6cm}}c@{\hspace{0.4cm}}c}
\toprule
\multirow{2}{*}{\textbf{NAS Alg.}}
& \multicolumn{2}{c}{\textbf{CIFAR-10}}
& \multicolumn{2}{c}{\textbf{CIFAR-100}} \\
\cmidrule(lr){2-3} \cmidrule(lr){4-5}
& \textbf{Val. Acc. (\%)}
& \shortstack{\textbf{Search Cost}\\\textbf{(CPU hours)}}
& \textbf{Val. Acc. (\%)}
& \shortstack{\textbf{Search Cost}\\\textbf{(CPU hours)}} \\
\midrule
DARTS      & $81.93 \pm 0.31$ & $26.30 \pm 0.56$ & $49.69 \pm 0.86$ & $26.12 \pm 0.40$ \\
S-LP-DARTS & $84.94 \pm 0.31$ & $25.30 \pm 0.22$ & $55.51 \pm 0.39$ & $25.84 \pm 0.66$ \\
\tb{R-LP-DARTS} & $\tb{84.96} \pm \tb{0.47}$ & $26.11 \pm 0.42$ & $\tb{55.78} \pm \tb{0.92}$ & $25.40 \pm 0.23$ \\
\bottomrule
\end{tabular}}
\end{table}

\begin{figure}[htb]
\centering

\begin{minipage}{0.48\textwidth}
\centering
\includegraphics[width=\linewidth]{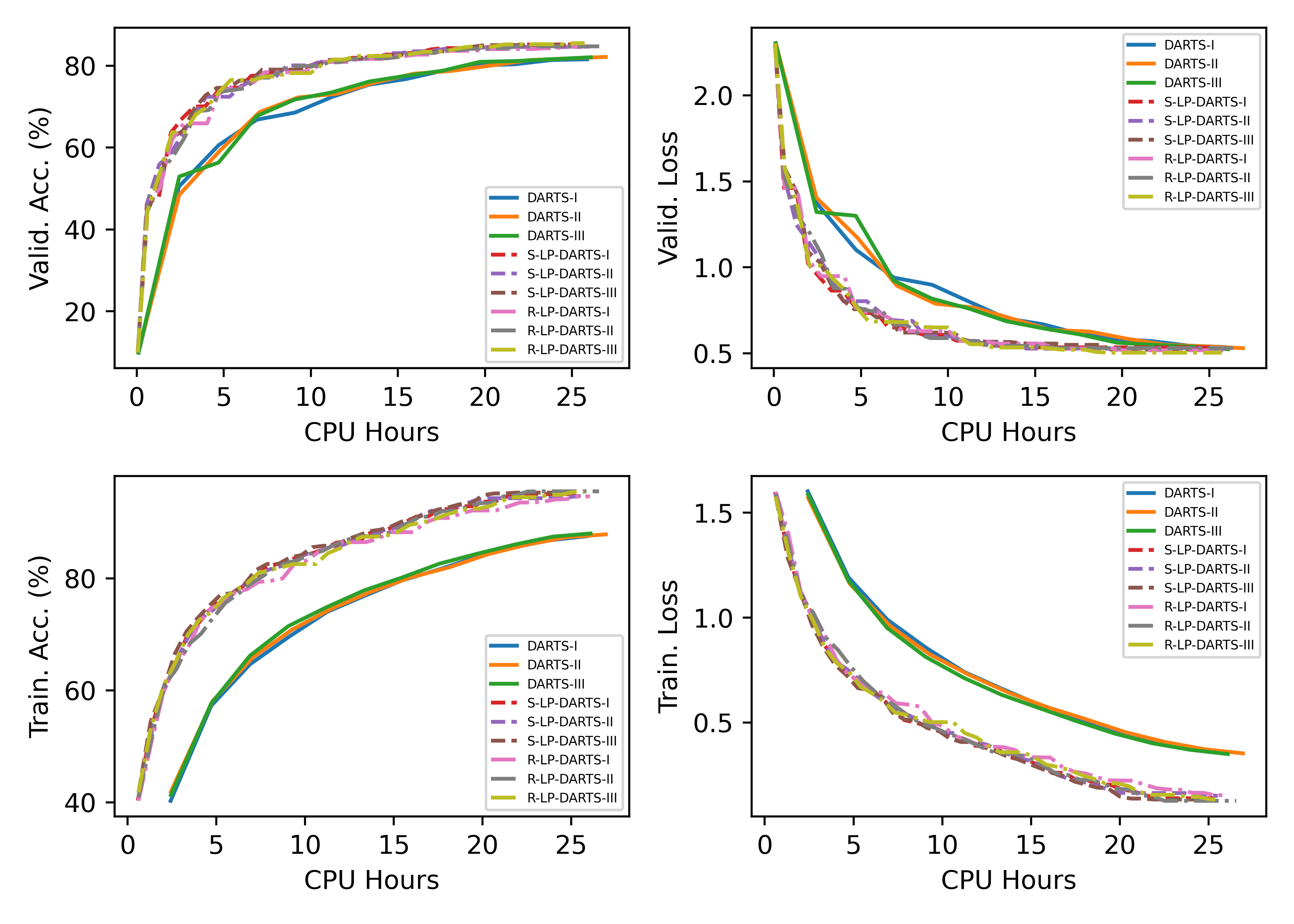}

(a) CIFAR-10
\end{minipage}
\hfill
\begin{minipage}{0.48\textwidth}
\centering
\includegraphics[width=\linewidth]{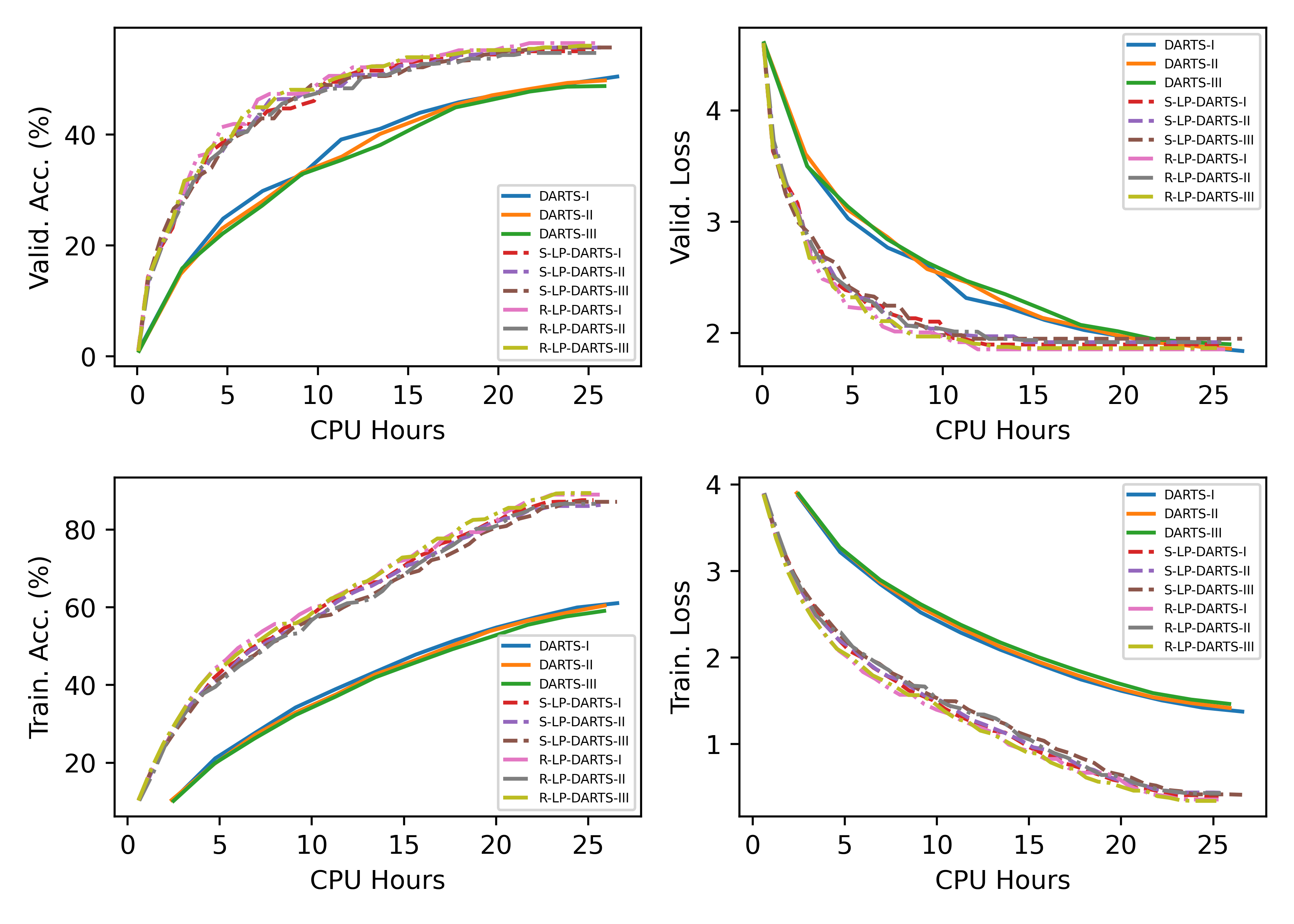}

(b) CIFAR-100
\end{minipage}

\caption{Architecture search performance plots of the algorithms for three independent runs on the CIFAR-10 and CIFAR-100 datasets.}
\label{fig:arch_search}
\end{figure}

S-LP-DARTS and R-LP-DARTS algorithms were executed three times on the CIFAR-10 and CIFAR-100 datasets, each with a different random seed. For comparison, the DARTS algorithm was also run three times using different random seeds. The performance trends observed during architecture optimization are illustrated in Figure~\ref{fig:arch_search}. The summarized results, including mean and standard deviation of validation accuracy and runtime, are presented in Table~\ref{tab:combined_nas_search}. As shown in Figure \ref{fig:arch_search}, both LP-DARTS approaches achieve improvements in validation accuracy much faster than DARTS. In all experiments, the proposed LP-DARTS approach outperforms DARTS in terms of validation accuracy.

\subsection{Architecture Evaluation}
We evaluate the searched architectures in two phases: the best architecture selection phase and the best architecture evaluation phase. The optimal normal and reduction cells are derived based on the highest validation accuracy achieved during the architecture search process. Once these optimal cells are identified, a larger model is constructed by stacking 20 optimal cells sequentially, with two reduction cells placed at 1/3 and 2/3 depths in the network (similar to the model considered for architecture search). The initial number of channels is increased from 16 to 32 to enhance the model's capacity. To improve generalization, a path dropout rate of 0.2 is applied, along with auxiliary towers weighted at 0.4 and cutout data augmentation. All other hyperparameters remain unchanged from those used during the architecture search. The larger models, built using the optimal cells found through the architecture search on CIFAR-10 and CIFAR-100, are trained from scratch on their respective full training datasets and evaluated on their respective test sets.

For the selection of the best architectures, the architectures derived from three architecture search experiments were trained for approximately 12 hours with a batch size of 576 and subsequently evaluated on the test datasets. The results, summarized in Table~\ref{tab:combined_eval_metrics}, consistently demonstrate the superior performance of the proposed method. Notably, models constructed using architectures searched by the two versions of LP-DARTS also exhibit significantly fewer parameters compared to those obtained with DARTS, as evidenced in Table~\ref{tab:combined_eval_metrics}. To illustrate the performance of the architectures in the evaluation phase, the test accuracy is plotted in Figure~\ref{fig:selection_eval}. Additionally, the visual representations of the best architectures (that is, yellow points in Figure~\ref{fig:selection_eval}) are provided in Appendix~\ref{app8}. To further verify the performance of the architectures, we once again train the model parameters for each of architectures (yellow points in Figure~\ref{fig:selection_eval}) for an extended period of time (that is 20 hours) and evaluate the models once again on the test dataset. 
The results of this final evaluation are presented in Table~\ref{tab:combined_cifar10_cifar100}. The training and validation performance of the models corresponding to the best architectures are presented in Figure~\ref{fig:selection_eval}. 

\begin{table}[ht]
\centering
\caption{Evaluation metrics for the architectures on CIFAR-10 and CIFAR-100.}
\label{tab:combined_eval_metrics}
\footnotesize
\resizebox{\textwidth}{!}{
\begin{tabular}{lc@{\hspace{0.25cm}}c@{\hspace{0.25cm}}c@{\hspace{0.5cm}}c@{\hspace{0.25cm}}c@{\hspace{0.25cm}}c}
\toprule
\multirow{2}{*}{\textbf{NAS Alg.}}
& \multicolumn{3}{c}{\textbf{CIFAR-10}}
& \multicolumn{3}{c}{\textbf{CIFAR-100}} \\
\cmidrule(lr){2-4} \cmidrule(lr){5-7}
& \textbf{Test Acc. (\%)}
& \shortstack{\textbf{Eval. Cost}\\\textbf{(CPU hours)}}
& \textbf{Params (M)}
& \textbf{Test Acc. (\%)}
& \shortstack{\textbf{Eval. Cost}\\\textbf{(CPU hours)}}
& \textbf{Params (M)} \\
\midrule
DARTS
& $70.83 \pm 2.09$
& $11.55 \pm 0.48$
& $3.40 \pm 0.10$
& $23.84 \pm 9.48$
& $11.85 \pm 0.44$
& $3.64 \pm 0.12$ \\

\tb{S-LP-DARTS}
& $89.04 \pm 1.26$
& $11.43 \pm 0.24$
& $1.41 \pm 0.22$
& $\tb{63.94} \pm \tb{3.61}$
& $11.78 \pm 0.18$
& $1.59 \pm 0.23$ \\

\tb{R-LP-DARTS}
& $\tb{89.87} \pm \tb{0.82}$
& $12.08 \pm 0.18$
& $1.66 \pm 0.28$
& $58.92 \pm 4.56$
& $12.02 \pm 0.30$
& $1.75 \pm 0.41$ \\
\bottomrule
\end{tabular}}
\end{table}

\begin{figure}[htb]
\centering

\begin{minipage}{0.49\textwidth}
\centering
\includegraphics[width=\linewidth]{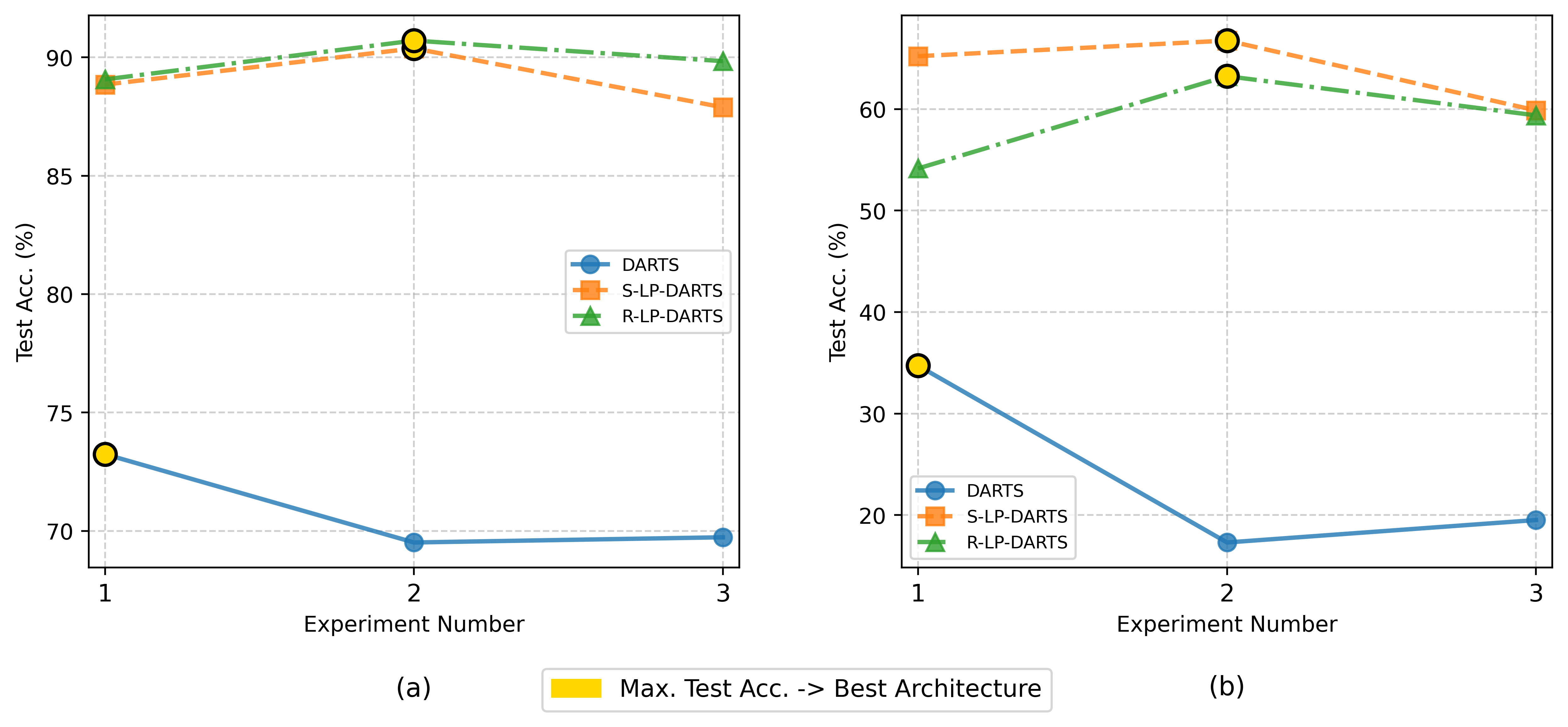}

(a) Architecture Selection
\end{minipage}
\hfill
\begin{minipage}{0.49\textwidth}
\centering
\includegraphics[width=\linewidth]{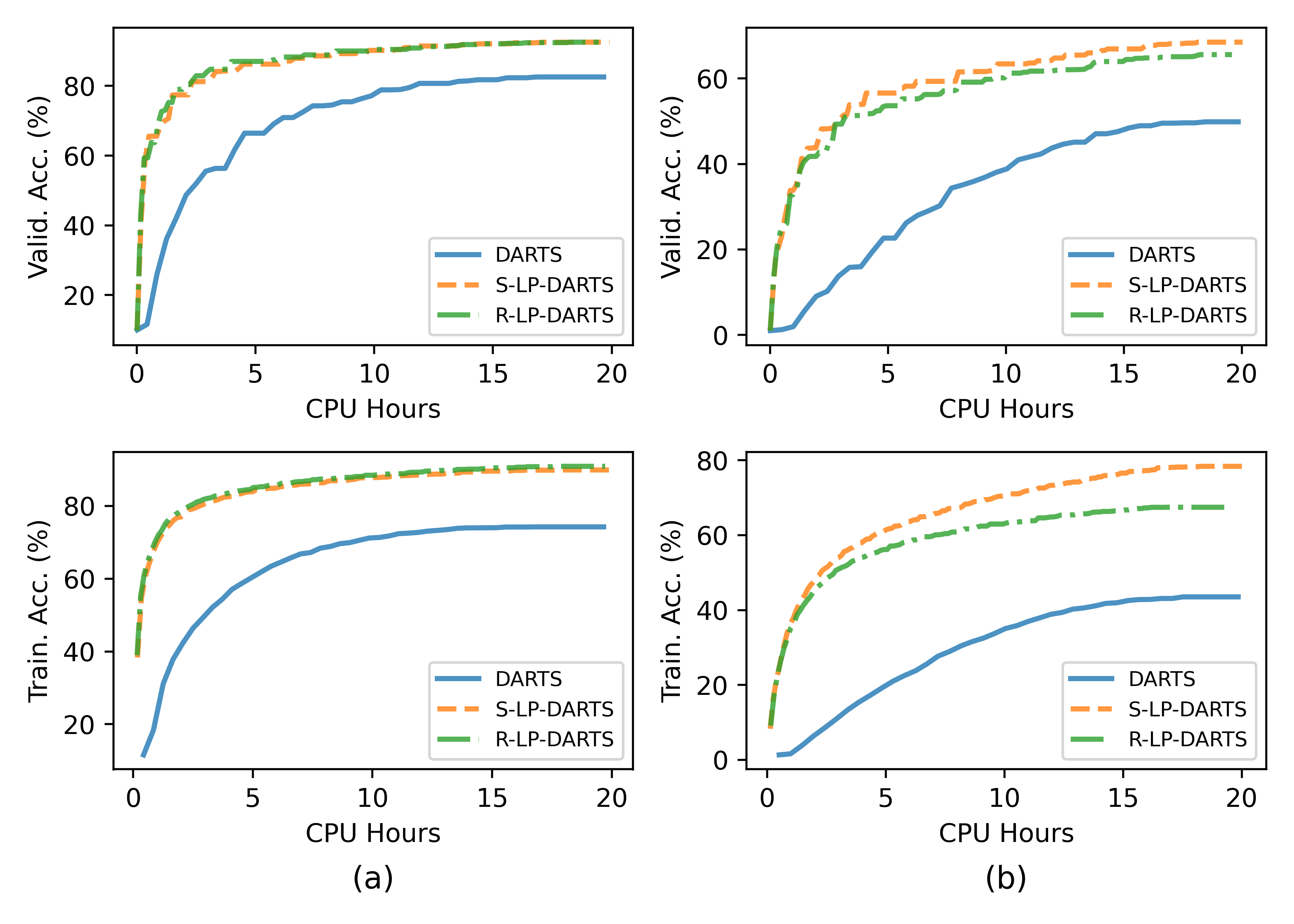}

(b) Architecture Evaluation
\end{minipage}

\caption{(a) Matrix plots illustrating architecture selection for the CIFAR-10 and CIFAR-100 datasets. (b) Training and validation performance of the best architectures on the CIFAR-10 and CIFAR-100 datasets.}
\label{fig:selection_eval}
\end{figure}

\begin{table}[ht]
\centering
\caption{Evaluation metrics for the best architectures on CIFAR-10 and CIFAR-100.}
\label{tab:combined_cifar10_cifar100}
\footnotesize
\resizebox{\textwidth}{!}{
\begin{tabular}{lc@{\hspace{0.25cm}}c@{\hspace{0.25cm}}c@{\hspace{0.5cm}}c@{\hspace{0.25cm}}c@{\hspace{0.25cm}}c}
\toprule
\multirow{2}{*}{\textbf{Method}} 
& \multicolumn{3}{c}{\textbf{CIFAR-10}} 
& \multicolumn{3}{c}{\textbf{CIFAR-100}} \\
\cmidrule(lr){2-4} \cmidrule(lr){5-7}
 & \textbf{Test Acc. (\%)} 
 & \shortstack{\textbf{Eval. Cost}\\\textbf{(CPU hours)}} 
 & \textbf{Params (M)}
 & \textbf{Test Acc. (\%)} 
 & \shortstack{\textbf{Eval. Cost}\\\textbf{(CPU hours)}} 
 & \textbf{Params (M)} \\
\midrule
DARTS       & 82.57 & 19.66 & 3.38 & 49.86 & 19.85 & 3.50 \\
\tb{S-LP-DARTS}  & \tb{92.56} & 19.91 & 1.66 & \tb{68.50} & 20.04 & 1.33 \\
R-LP-DARTS  & 92.55 & 19.72 & 1.36 & 65.58 & 19.59 & 1.36 \\
\bottomrule
\end{tabular}}
\end{table}

While both variants of LP-DARTS significantly outperformed DARTS in the overall evaluation results, the experimental findings did not indicate a definitive superiority between the two variants of LP-DARTS, as illustrated in Figure~\ref{fig:selection_eval}. However, the best architectures discovered by S-LP-DARTS achieved superior performance compared to those identified by R-LP-DARTS across both datasets, as presented in Table~\ref{tab:combined_cifar10_cifar100}.

\subsection{Architecture Transferability}\label{sec:comparative}
To evaluate the transferability of the learned architectures, we assess models on ImageNet using architectures obtained from CIFAR-10 via several NAS methods, including DARTS, LP-DARTS, P-DARTS, PC-DARTS, and STO-DARTS. Each model comprises 14 cells, with reduction cells placed at one-third and two-thirds of the total depth, and an initial channel size of 48. All models are trained from scratch on ImageNet for approximately 100 hours using a batch size of 576, an initial learning rate of 0.1, and a decay factor of 0.97, closely following the original DARTS evaluation protocol. The performance of these models is summarized in Table~\ref{tab:combined_cifar10_imagenet}. The results indicate that LP-DARTS achieves competitive transferability performance from CIFAR-10 to ImageNet while maintaining the lowest model complexity among the considered methods.

Notably, the architectures for DARTS and LP-DARTS are obtained under a controlled and comparable search budget, with all computations performed on CPUs. These architectures are selected during the early stages of the search process-well before full convergence-to reflect a realistic, resource-constrained setting. In contrast, the architectures for P-DARTS, PC-DARTS, and STO-DARTS variants are adopted from their original works~\cite{chen2019progressive, xu2019pcdarts, cai2024sto}, where the search is carried out to convergence using significantly greater computational resources, typically on GPUs. This distinction underscores the efficiency and practical viability of LP-DARTS in identifying high-quality, transferable architectures under limited computational budgets. Overall, despite operating under comparable or more restrictive search and evaluation conditions, LP-DARTS achieves competitive test performance with highest parameter efficiency, demonstrating their effectiveness in discovering competitive and transferable architectures.

\begin{table}[htb]
\centering
\caption{Evaluation metrics for the best architectures on CIFAR-10 and their transferability on ImageNet under time-constrained settings.}
\label{tab:combined_cifar10_imagenet}
\footnotesize
\resizebox{\textwidth}{!}{
\begin{tabular}{lc@{\hspace{0.25cm}}c@{\hspace{0.25cm}}c@{\hspace{0.5cm}}c@{\hspace{0.25cm}}c@{\hspace{0.25cm}}c}
\toprule
\multirow{2}{*}{\textbf{Method}} & \multicolumn{3}{c}{\textbf{CIFAR-10}} & \multicolumn{3}{c}{\textbf{ImageNet}} \\
\cmidrule(lr){2-4} \cmidrule(lr){5-7}
 & \textbf{Test Acc. (\%)}
 & \shortstack{\textbf{Eval. Cost}\\\textbf{(CPU hours)}} 
 & \textbf{Params (M)}
 & \textbf{Top-1 (5) Acc. (\%)} 
 & \shortstack{\textbf{Eval. Cost}\\\textbf{(CPU hours)}} 
 & \textbf{Params (M)} \\
\midrule
DARTS & 82.57 & 19.66 & 3.38 & 39.09 (64.10) & 99.36 & 5.81 \\
P-DARTS & 90.23 & 21.67 & 2.74 & 44.60 (69.38) & 98.96 & 4.94 \\
PC-DARTS & 89.10 & 20.84 & 2.90 & 42.40 (67.39) & 101.53 & 5.27 \\
STO-DARTSv1 & 89.04 & 19.79 & 2.20 & 42.72 (67.78) & 101.52 & 4.15 \\
STO-DARTSv2 & 87.89 & 20.85 & 3.06 & 42.89 (68.05) & 102.32 & 5.37 \\
\tb{LP-DARTS (ours)} & \tb{92.56} & 19.91 & 1.66 & 42.50 (67.55) & 99.32 & 3.46 \\
\bottomrule
\end{tabular}}
\end{table}

\subsection{Analysis and Discussion}
The comparative results summarized from the NAS literature \cite{zoph2018learning, pham2018efficient, real2019regularized, DBLP:conf/iclr/LiuSY19, xu2019pcdarts, cai2018proxylessnas, wu2019fbnet, tan2019mnasnet, tan2019efficientnet, cai2019once, chen2021autoformer, gao2019graphnas, chen2019progressive, cai2024sto} (Appendix~\ref{app9}) reveal several noteworthy trends. Early NAS methods such as NASNet and AmoebaNet, which rely on RL or evolutionary strategies, achieve strong accuracy but demand thousands of GPU-days due to discrete search and full training of candidate architectures. The introduction of differentiable NAS methods such as DARTS and PC-DARTS marked a breakthrough in efficiency, reducing computational cost by three orders of magnitude while maintaining competitive accuracy. Hardware-aware designs such as ProxylessNAS, FBNet, and MnasNet further refined this efficiency by explicitly incorporating latency constraints during search, making them viable for mobile and embedded systems. Meta-learning approaches, such as Once-for-All and Autoformer, illustrate a shift toward scalable and adaptive NAS frameworks that can generalize across architectures and modalities. Once-for-All introduces elastic supernetworks, allowing the reuse of pre-trained subnetworks across multiple deployment scenarios, while Autoformer demonstrates that NAS principles extend beyond CNNs to transformer-based architectures. Collectively, these methods highlight an ongoing trend toward generalizable and hardware-efficient NAS. 

Table~\ref{tab:combined_cifar10_imagenet} shows that LP-DARTS achieves superior performance compared to several DARTS variants under the time-constrained experimental settings on CIFAR-10, and remains competitive in the architecture transferability evaluation on ImageNet. This is particularly noteworthy given that the architectures for P-DARTS, PC-DARTS, and STO-DARTS are adopted from their original studies, where extensive experimentation and greater computational resources were used to identify optimal configurations. Furthermore, LP-DARTS attains the lowest model complexity among the considered methods.

The superior performance of the proposed method arises from the inherent strength of its bilevel optimization framework. Unlike conventional DARTS variants, which approximate the hypergradient using only partial information, our LP formulation explicitly enforces lower-level optimality at each update. In particular, the solution of the LP provides a coupled direction $(d_A^\ast, d_W^\ast)$, where $d_A^\ast$ denotes the descent direction for the architecture variables and $d_W^\ast$ represents the consistent adjustment of network weights that preserves training optimality. This joint update mechanism effectively captures the interdependence between the upper-level validation objective and the lower-level training problem, thereby avoiding the instability that often limits DARTS-style methods, which do not ensure the optimality of the model parameters after architecture updates. Also, it can be easily shown that the original SOCP, which computes the steepest descent direction through its closed-form solution, yields the same direction for the architecture parameters as hypergradient-based methods \cite{giovannelli2025inexact}. Notably, exact hypergradient computation requires the inversion of the training-loss Hessian, whereas the proposed LP-based NAS framework operates directly on the Hessian, thereby avoiding this computationally expensive operation. More importantly, the proposed formulation is not restricted to the classical DARTS search space and is broadly applicable to a wide range of differentiable NAS settings, making it relevant to expert system applications based on deep learning.

\section{Conclusion}\label{sec:conclusions}
In this paper, we have introduced LP-NAS, an approach for architecture search within the bilevel optimization framework. The approach requires an LP to be solved in each iteration to obtain a descent direction for validation loss while ensuring that the training loss remains optimal. Applied to CNNs in DSS, LP-NAS surpasses the performance of DARTS, discovering cell architectures that lead to superior results. It also outperforms several DARTS variants under time-constrained experimental settings on the CIFAR-10 dataset while remaining competitive in architecture transferability to the ImageNet dataset.

A key challenge in LP-NAS lies in the calculation of the Hessian matrix and solving the LP. To cater to the computational issues associated with these challenges we use a reduced Hessian, approximate it using L-BFGS, and formulate a reduced LP. The use of limited second-order information from the lower-level is still advantageous as the method quickly leads to better losses and accuracies in all our experiments as compared to the standard DARTS algorithm. The proposed LP based approach can be integrated with any architecture search problem with continuous architecture and model parameters. Interesting future research can be performed on appropriate choice of model parameters for the formulation of the reduced Hessian such that maximum amount of second-order information is extracted. Additionally, approximate Hessians with special structures, such as, block diagonal, block angular, block triangular, staircase, and bordered angular, can be formulated for faster solution of the LP. The results are promising and we believe it has the potential to open a new line of research on solving architecture search problems using bilevel methods by utilizing second-order information.

\section*{Acknowledgments}
The authors gratefully acknowledge the High Performance Computing (HPC) facility at the Indian Institute of Technology (IIT) Kanpur (PARAM Sanganak) for providing the resources that enabled the smooth execution of the experiments in this work.

\bibliography{main}

\appendix
\renewcommand{\thesubsection}{\Alph{subsection}}
\renewcommand{\thesubsubsection}{\Alph{subsection}.\arabic{subsubsection}}

\section*{\centering \Huge Appendix}

\section{DARTS} \label{app1}
The DARTS method, introduced in \cite{DBLP:conf/iclr/LiuSY19}, presented a highly efficient and competitive NAS approach compared to RL- and EC-based methods, functioning within a continuous search space. Building on DARTS, Progressive Differentiable Architecture Search (P-DARTS) was developed by \cite{chen2019progressive}, which addressed computational challenges and enhanced search stability through search space approximation and regularization techniques. To further improve the robustness of DARTS, \cite{zela2019understanding} proposed Robustness of Differentiable Architecture Search (R-DARTS), which systematically investigated architectural spaces and regularization strategies. In \cite{xu2019pcdarts}, the authors introduced an innovative method known as Partially-Connected Differentiable Architecture Search (PC-DARTS) with the obvious goal of improving the efficiency and stability of NAS. Another extension, Att-DARTS~\cite{nakai2020att}, incorporated attention modules into the DARTS framework to further enhance the architecture search process. The manifestation of the collapse phenomenon in DARTS, characterized by an excessive occurrence of skip-connects over an extensive number of search epochs due to overfitting of the one-shot model, resulted in diminished performance. This challenge was effectively addressed by introducing early stopping criteria in DARTS+ \cite{liang2019darts+}, where DARTS stops when two or more skip-connects appear in a normal cell or when the architecture parameter ranking remains stable for a predefined number of epochs. Fair DARTS \cite{fair-darts} was introduced to tackle issues of collapse observed in the DARTS algorithm due to unfair advantages in exclusive competition among candidate operations during the search process. DE-DARTS developed by \cite{de-darts} addressed the challenges of gradient-based NAS by proposing a novel approach that incorporated Dynamic Attention Networks (DANs). More recently, EG-DARTS \cite{zhang2023enhanced} used a multi-objective evolution-based approach, combining gradient optimization with evolutionary strategies to improve DARTS’ search effectiveness. Additional advancements have been made with methods like Relax-DARTS \cite{zhu2024relax}, STO-DARTS \cite{cai2024sto}, OSTR-DARTS \cite{yang2024ostr}, semantic DARTS \cite{guo2024semantic}, LMD-DARTS \cite{li2024lmd}, HN-DARTS \cite{li2024hn}, and ZEN-DARTS \cite{wang2026zen}, among others, highlighting continued innovation within the DARTS framework.

\section{LP for HLS} \label{app2}
Assuming that for any given set of hyperparameters, a solution to the lower-level problem always exists. Additionally, it is required that the validation loss function be at least once differentiable and the training loss function be at least twice differentiable.

By writing the linear Taylor's approximation of the validation loss function around the point \((A^0, W^0)\), we obtain the following expansion:
\begin{equation*}
    \begin{aligned}
\mathcal{L}_v(A^0 + t d_A, W^0 + t d_W)=\\
& \hspace{-3cm} \mathcal{L}_v(A^0, W^0) + t \left \langle \nabla_A \mathcal{L}_v(A^0, W^0), d_A \right\rangle\\
& \hspace{-3cm} + t\left \langle \nabla_W \mathcal{L}_v(A^0, W^0), d_W \right\rangle\\
\end{aligned}
\end{equation*}

where \(t\) is the step-size greater than zero and $[d_A, d_W]^T$ is a direction vector.

\begin{equation*}
    \begin{aligned}
\mathcal{L}_v(A^0 + t d_A, W^0 + t d_W)-\mathcal{L}_v(A^0, W^0)=\\
& \hspace{-3cm}  t\left\langle \begin{bmatrix}
    \nabla_A \mathcal{L}_v(A^0, W^0) \\
    \nabla_W \mathcal{L}_v(A^0, W^0)
\end{bmatrix}, \begin{bmatrix}
    d_A \\
    d_W
\end{bmatrix} \right\rangle
\end{aligned}
\end{equation*}

From the above equation, it can be concluded that the direction vector is a descent direction (validation loss improves) if and only if the inner product of the gradient of the upper-level objective function with the direction vector is negative. This can be mathematically expressed as  
\[
\left\langle \begin{bmatrix}
    \nabla_A \mathcal{L}_v(A^0, W^0) \\
    \nabla_W \mathcal{L}_v(A^0, W^0)
\end{bmatrix}, \begin{bmatrix}
    d_A \\
    d_W
\end{bmatrix} \right\rangle < 0
\]
where \(\nabla_A \mathcal{L}_v\) and \(\nabla_W \mathcal{L}_v\) represent the gradients of the validation loss function with respect to \(A\) and \(W\) respectively. Moreover, this descent direction must ensure that the lower-level problem remains optimal as the upper-level variable \(A\) changes along the direction \(d_A\). This imposes the condition  
\begin{equation}\label{lowerleveldir}
    d_W \in \arg\min_{d_W} \mathcal{L}_t(A^0 + t d_A, W^0 + t d_W)
\end{equation}

which ensures that \(d_W\) is the direction minimizing the lower-level objective with respect to \(W\), given the change in \(A\) along the direction \(d_A\).

By writing the quadratic approximation of the lower-level objective around the point \((A^0, W^0)\), we obtain the following expansion:
\begin{equation}\label{Taylor}
    \begin{aligned}
\mathcal{L}_t(A^0 + t d_A, W^0 + t d_W)=\\
& \hspace{-3cm} \mathcal{L}_t(A^0, W^0) + t \left \langle \nabla_A \mathcal{L}_t(A^0, W^0), d_A \right\rangle\\
& \hspace{-3cm} + t\left \langle \nabla_W \mathcal{L}_t(A^0, W^0), d_W \right\rangle\\
& \hspace{-3cm} + \frac{1}{2}t^2\left\langle \begin{bmatrix}
    d_A \\
    d_W
\end{bmatrix}, \nabla_{(A, W)} ^2 \mathcal{L}_t(A^0, W^0)\begin{bmatrix}
    d_A \\
    d_W
\end{bmatrix}  \right\rangle
\end{aligned}
\end{equation}

where \(\nabla_{(A, W)}^2 \mathcal{L}_t(A^0, W^0)\) denotes the Hessian matrix of the training loss function \((\mathcal{L}_t)\) with respect to \(A\) and \(W\), evaluated at \((A^0, W^0)\).

At the optimal solution of the lower-level objective, the second and third terms vanish. Ignoring the constant first term, the lower-level problem \eqref{lowerleveldir} reduces to the following optimization problem:
\begin{equation}\label{lowerleveldir_1}
    d_W \in \arg\min_{d_W} \left\langle 
    \begin{bmatrix}
        d_A \\
        d_W
    \end{bmatrix}, 
    \nabla_{(A, W)}^2 \mathcal{L}_t(A^0, W^0) 
    \begin{bmatrix}
        d_A \\
        d_W
    \end{bmatrix}  
    \right\rangle
\end{equation}

Let \(p\) and \(q\) denote the number of hyperparameters and model parameters, respectively. The Hessian matrix of the training loss can be expressed as follows:
\begin{equation}\label{hessian_matrix}
\begin{bmatrix}
    H_{ij}
\end{bmatrix}_{i=1, j=1}^{i=p+q, j=p+q} = \nabla_{(A, W)}^2 \mathcal{L}_t(A^0, W^0)
\end{equation}

Let
\begin{equation}\label{hessian_matrices}
\begin{aligned}
\begin{bmatrix}
    H_{ij}
\end{bmatrix}_{i=1, j=1}^{i=p+q, j=p+q} 
&= \begin{bmatrix}
    [H_{ij}]_{i=1, j=1}^{i=p, j=p} & [H_{ij}]_{i=1, j=p+1}^{i=p, j=p+q} \\
    [H_{ij}]_{i=p+1, j=1}^{i=p+q, j=p} & [H_{ij}]_{i=p+1, j=p+1}^{i=p+q, j=p+q}
\end{bmatrix} \\
&= \begin{bmatrix}
    H_{11} & H_{12} \\
    H_{21} & H_{22}
\end{bmatrix}
\end{aligned}
\end{equation}

Since the problem \eqref{lowerleveldir_1} is an unconstrained optimization problem, the first-order optimality conditions can be derived by taking the gradient of the objective function with respect to \(d_W\) and setting it equal to zero, which reduces to:
\begin{equation}\label{firstordercondition}
\begin{aligned}
    \begin{bmatrix}
        H_{21} & H_{22}
    \end{bmatrix} 
    \begin{bmatrix}
        d_A \\
        d_W
    \end{bmatrix} = 0 \\
\end{aligned}
\end{equation}
Ensuring the lower-level optimality conditions, if we aim to find the unit vector with the steepest descent direction for validation loss, it leads to solving the following SOCP problem:

\begin{equation}\label{devised_SOCP}
	\begin{aligned}
		\min_{d_A, d_W} \quad & \left\langle \begin{bmatrix}
    \nabla_A \mathcal{L}_v(A^0, W^0) \\
    \nabla_W \mathcal{L}_v(A^0, W^0)
\end{bmatrix}, \begin{bmatrix}
    d_A \\
    d_W
\end{bmatrix} \right\rangle \\
		\text{subject to:} \quad & \begin{bmatrix}
        H_{21} & H_{22}
    \end{bmatrix} 
    \begin{bmatrix}
        d_A \\
        d_W
    \end{bmatrix} = 0\\
		\quad & \|d_A\|_2 \leq 1
	\end{aligned}
\end{equation}

This SOCP problem can be relaxed to the following usable LP to give a descent direction:

\begin{equation}
	\begin{aligned}
		\min_{d_A, d_W} \quad & \left\langle \begin{bmatrix}
    \nabla_A \mathcal{L}_v(A^0, W^0) \\
    \nabla_W \mathcal{L}_v(A^0, W^0)
\end{bmatrix}, \begin{bmatrix}
    d_A \\
    d_W
\end{bmatrix} \right\rangle \\
		\text{subject to:} \quad & \begin{bmatrix}
        H_{21} & H_{22}
    \end{bmatrix} 
    \begin{bmatrix}
        d_A \\
        d_W
    \end{bmatrix} = 0\\
		\quad & -1 \leq d_A \leq 1 \\
	\end{aligned}
\end{equation}

\section{Proof of Proposition~1} \label{app3}

\begin{proof}
The affine (linear) constraints of the LP can be written as:
\begin{equation}\label{affine_constraints}
	\begin{aligned}
         H_{21}d_A + H_{22}d_W
        = 0 \\
	\end{aligned}
\end{equation}

From equations \eqref{Taylor}, \eqref{hessian_matrix} and \eqref{hessian_matrices}, we have
\begin{equation}\label{Taylor_1}
    \begin{aligned}
\mathcal{L}_t(A^0 + t d_A, W^0 + t d_W)=\\
& \hspace{-3cm} \mathcal{L}_t(A^0, W^0) + t \left \langle \nabla_A \mathcal{L}_t(A^0, W^0), d_A \right\rangle\\
& \hspace{-3cm} + t\left \langle \nabla_W \mathcal{L}_t(A^0, W^0), d_W \right\rangle\\
& \hspace{-3cm} + \frac{1}{2}t^2\left\langle \begin{bmatrix}
    d_A \\
    d_W
\end{bmatrix}, \begin{bmatrix}
    H_{11} & H_{12} \\
    H_{21} & H_{22}
\end{bmatrix} \begin{bmatrix}
    d_A \\
    d_W
\end{bmatrix}  \right\rangle
\end{aligned}
\end{equation}

From equations \eqref{affine_constraints} and \eqref{Taylor_1}
\begin{equation}\label{Taylor_3}
    \begin{aligned}
\mathcal{L}_t(A^0 + t d_A, W^0 + t d_W)=\\
& \hspace{-3cm} \mathcal{L}_t(A^0, W^0) + t \left \langle \nabla_A \mathcal{L}_t(A^0, W^0), d_A \right\rangle\\
& \hspace{-3cm} + t\left \langle \nabla_W \mathcal{L}_t(A^0, W^0), d_W \right\rangle\\
& \hspace{-4cm} + \frac{1}{2}t^2\left\langle 
    d_A, H_{11}d_A + H_{12}d_W \right\rangle
\end{aligned}
\end{equation}

Writing \(d_W\) in terms of \(d_A\) and Hessian matrices using equation \eqref{affine_constraints}

\begin{equation}\label{d_A_d_W}
   d_W = -H_{22}^{-1} H_{21}d_A 
\end{equation}

Building upon Equations~\eqref{Taylor_3} and \eqref{d_A_d_W}
\begin{equation}\label{Taylor_4}
    \begin{aligned}
\mathcal{L}_t(A^0 + t d_A, W^0 + t d_W)=\\
& \hspace{-3cm} \mathcal{L}_t(A^0, W^0) + t \left \langle \nabla_A \mathcal{L}_t(A^0, W^0), d_A \right\rangle\\
& \hspace{-3cm} + t\left \langle \nabla_W \mathcal{L}_t(A^0, W^0), d_W \right\rangle\\
& \hspace{-4cm} + t^2\left\langle 
    d_A, (H_{11} - H_{12}H_{22}^{-1}H_{21})d_A \right\rangle
\end{aligned}
\end{equation}

Assuming the pseudo-inverse of \(H_{21}\), denoted as \(H_{21}^\dagger\), we can express \(d_A\) in terms of \(d_W\) as
\begin{equation}\label{d_W_d_A}
    d_A = -H_{21}^\dagger H_{22} d_W
\end{equation}

Substituting \(d_A\) from Equation~\eqref{d_W_d_A} into Equation~\eqref{Taylor_4}, we rewrite the Taylor's expansion in terms of \(d_W\)

\begin{equation*}\label{Taylor_5}
    \begin{aligned}
\mathcal{L}_t(A^0 + t d_A, W^0 + t d_W)=\\
& \hspace{-3cm} \mathcal{L}_t(A^0, W^0) + t \left \langle \nabla_A \mathcal{L}_t(A^0, W^0), -H_{21}^\dagger H_{22} d_W \right\rangle\\
& \hspace{-3cm} + t\left \langle \nabla_W \mathcal{L}_t(A^0, W^0), d_W \right\rangle\\
& \hspace{-4cm} + t^2\left\langle 
    H_{21}^\dagger H_{22} d_W, (H_{11} - H_{12}H_{22}^{-1}H_{12}^T)H_{21}^\dagger H_{22} d_W \right\rangle
\end{aligned}
\end{equation*}

The second-order gradient with respect to the updated model parameters, \((W^0 + t d_W)\), yields the new Hessian matrix
\begin{equation*}\label{newH}
   H'' = (H_{21}^\dagger H_{22})^T(H_{11} - H_{12}H_{22}^{-1}H_{12}^T)H_{21}^\dagger H_{22} 
\end{equation*}

where \(S = H_{11} - H_{12}H_{22}^{-1}H_{12}^T\) represents the Schur complement of \(H_{22}\) in the original Hessian matrix
$\begin{bmatrix} H_{ij} \end{bmatrix}_{i=1, j=1}^{i=p+q, j=p+q}$.

If the original Hessian matrix is positive semidefinite (PSD), its Schur complement \(S\) is also guaranteed to be PSD. \(H''\) will be PSD if for all vectors \(z \in \mathbb{R}^{q}\), \(z^TH''z \geq 0\). To verify the positive semidefiniteness of \(H''\), we consider an arbitrary vector \(y \in \mathbb{R}^q\)
\[
y^T H'' y = y^T \big((H_{21}^\dagger H_{22})^T S (H_{21}^\dagger H_{22})\big) y
\]
Let \(v = H_{21}^\dagger H_{22} y\), then $y^T H'' y = v^T S v$. Since \(S\) is PSD, it follows that \(v^T S v \geq 0\) for all \(v \in \mathbb{R}^p\). Thus,
\[
y^T H'' y \geq 0 \quad \forall y \in \mathbb{R}^q
\]
which confirms that \(H''\) is PSD. In conclusion, the new Hessian matrix \(H''\) being PSD establishes that the HLS is confined to the lower-level optimality region.
\end{proof}

\section{Proof of Proposition~2} \label{app4}
\begin{proof}
To establish that the optimal objective value is always non-positive, we construct a feasible solution and evaluate its objective value.\\

\noindent \textit{Feasibility of \( (d_A, d_W) = (0,0) \)}:
Choosing \( d_A = 0 \) and \( d_W = 0 \), the affine constraints of the devised LP are trivially satisfied
\begin{align*}
    H_{21} (0) + H_{22} (0) = 0
\end{align*}
Additionally, the box constraints hold as
\begin{align*}
    -1 \leq 0 \leq 1
\end{align*}

\noindent \textit{Objective Function Evaluation}:
At \( (d_A, d_W) = (0,0) \), the objective function evaluates to
\begin{align*}
    \left\langle 
    \begin{bmatrix} 
        \nabla_A \mathcal{L}_v(A^0, W^0) \\
        \nabla_W \mathcal{L}_v(A^0, W^0) 
    \end{bmatrix}, 
    \begin{bmatrix} 
        0 \\ 
        0 
    \end{bmatrix} 
    \right\rangle = 0
\end{align*}
Since the problem is of minimization type, the optimal objective value satisfies
\begin{align*}
    \left\langle 
    \begin{bmatrix} 
        \nabla_A \mathcal{L}_v(A^0, W^0) \\
        \nabla_W \mathcal{L}_v(A^0, W^0) 
    \end{bmatrix}, 
    \begin{bmatrix} 
        d_A^* \\ 
        d_W^* 
    \end{bmatrix} 
    \right\rangle \leq 0
\end{align*}
Thus, the validation performance improves in most cases, and in the worst case, leaves it unchanged.
\end{proof}

\section{Reduced Hessian Formulation for Memory and Computational Efficiency} \label{app5}

In deep learning, the model architecture is characterized by architecture parameters and model parameters. The architecture parameters typically have a much smaller dimension than the model parameters, i.e., $p \ll q$, whereas the number of model parameters can range from millions to billions. Consequently, formulating the LP presents a significant computational challenge due to the size of the Hessian matrix.

For a model with $p$ architecture parameters and $q$ model parameters, the complete Hessian of the training loss has dimension $(p+q)\times(p+q)$. Since the lower-level optimization is performed only with respect to the model parameters, only the submatrix of size $q\times(p+q)$ is required by the proposed LP. For convenience, we continue to refer to this rectangular submatrix as the Hessian matrix throughout the paper.

The memory required to store this matrix is substantial. For example, a model with approximately two million parameters requires nearly $29.8$~TB of memory to store the Hessian in double precision. Furthermore, the computational complexity of explicit Hessian computation is $O(n^2)$, making it prohibitively expensive for modern deep neural networks.

To overcome these limitations, we compute second-order derivatives only for a subset of the model parameters, denoted by $w\subset W$. Specifically, all $p$ architecture parameters together with only $\kappa q$ model parameters are retained, where $\kappa\ll1$. The resulting reduced Hessian,
$[H'_{ij}]_{i=p+1,\;j=1}^{\,i=p+\kappa q,\;j=p+\kappa q},$ has dimension $\kappa q \times (p+\kappa q).$ Similarly, the validation gradient is computed only with respect to the retained parameters, i.e. $[\nabla_A\mathcal{L}_v(A,W),\;
\nabla_w\mathcal{L}_v(A,W)]^{T},$ which we refer to as the \emph{reduced validation gradient}. These quantities are then used to formulate a reduced LP that can be solved efficiently.

Finally, we employ the L-BFGS algorithm to approximate the reduced Hessian. L-BFGS requires only $O(mn)$ computation per iteration, where $m$ denotes the memory parameter (the number of stored correction pairs) and $n$ is the parameter dimension. Consequently, the proposed formulation achieves substantial reductions in both memory consumption and computational time, at the cost of utilizing only partial second-order information from the lower-level problem.

\section{Mathematical Formulation of Model Parameter Selection} \label{app6}

This section presents the mathematical formulation of the two model parameter selection strategies used for reduced Hessian computation in LP-NAS.

In deep learning models, model parameters are typically represented as multi-dimensional tensors. In our experimental configurations, which employ a convolutional neural network (CNN), most parameters are structured as four-dimensional tensors. Let the $i^{\text{th}}$ model parameter tensor be denoted by $\mathbf{W}_i$, with $N_i$ dimensions such that
\[
\mathbf{W}_i \in \mathbb{R}^{d_1^i \times d_2^i \times \cdots \times d_{N_i}^i},
\]
where $d_j^i$ denotes the size of the $j^{\text{th}}$ dimension.

\paragraph{Specific parameter selection (S-LP-NAS).}

Following the lower-level optimization, the gradients of the training loss with respect to all model parameters are available. These gradients are used to compute the Gradient Norm Per Parameter (GNPP),

\begin{equation}
\label{gnpp}
\text{GNPP}_i =
\frac{\|\nabla_{\mathbf{W}_i}\mathcal{L}_t(A^0,W^0)\|_2}
{\prod_{j=1}^{N_i} d_j^i}.
\end{equation}

The model parameter tensor is selected according to

\begin{equation}
i^*=\arg\max_i \text{GNPP}_i,
\end{equation}

and the number of selected parameters is

\begin{equation}
\label{specific_selection}
\kappa q=\prod_{j=1}^{N_{i^*}} d_j^{\,i^*}.
\end{equation}

This criterion favors parameter tensors exhibiting large training-loss gradients while maintaining a relatively small tensor size.

\paragraph{Random parameter selection (R-LP-NAS).}

Alternatively, a subset of parameter tensors is selected uniformly at random from the set of valid candidate tensors satisfying a prescribed size constraint. Let

\begin{equation}
\label{valid_indices}
\mathcal{C}=
\left\{
i \,\middle|\,
\prod_{j=1}^{N_i} d_j^i
\le
\varkappa q
\right\},
\end{equation}

where $\varkappa q$ denotes the maximum allowable number of parameters for reduced Hessian computation.

A subset of $n$ tensors is sampled as

\begin{equation}
\label{index_set}
I=
\left\{
i_j
\mid
i_j\sim\mathrm{Uniform}(\mathcal{C}),
\;
j=1,\ldots,n
\right\},
\end{equation}

and the total number of selected parameters is

\begin{equation}
\label{random_selection}
\kappa q
=
\sum_{i\in I}
\prod_{j=1}^{N_i} d_j^i.
\end{equation}

Since different parameter tensors are selected across iterations, this strategy enables diverse Hessian approximations and provides update opportunities for different regions of the network.

\section{Time Complexity} \label{app7}
The computational cost of each outer iteration is summarized as follows:
\begin{enumerate}
    \item \textit{Hessian approximation (L-BFGS):} \( O(mn) \)
    \item \textit{Validation gradient and parameter update:} \( O(n) \)
    \item \textit{LP formulation and solution:} The LP involves \( n \) variables, \( \kappa q \) linear equality constraints, and \( p \) box constraints. While the worst-case complexity of generic LP solvers is \( O(n^3) \), the practical cost in our setting is significantly lower due to the small value of \( \kappa \) and the structured nature of the constraint matrix (primarily linear orthogonality constraints). Empirically, the solve time is well approximated by \( O(n^2) \).
    \item \textit{Lower-level weight optimization:} Processing \( E_D \) samples with batch size \( B \) requires \( O\!\left(\frac{E_D}{B} C_P \right) \).
\end{enumerate}
Aggregating these components over \( k^{\text{max}} \) outer iterations, the total time complexity is:
\begin{equation}
O\!\Bigg(
k^{\text{max}} \left[
mn + n + n^2 + \frac{E_D}{B} C_P
\right]
\Bigg).
\label{eq:complexity}
\end{equation}
Since typically \( m \ll n \) (e.g., \( m \approx 15 \) and \( n \approx 10^3 \) in our experiments), the linear term \( mn \) is dominated by the quadratic term \( n^2 \). Therefore, the overall complexity can be simplified as:
\begin{equation}
O\!\left(
k^{\text{max}} \left[
n^2 + \frac{E_D}{B} C_P
\right]
\right).
\end{equation}

\section{Optimal Normal and Reduction Cells} \label{app8}
The optimal normal and reduction cells discovered by the NAS methods for the CIFAR-10 and CIFAR-100 datasets are shown in Figures~\ref{cifar10_optimal_cells} and~\ref{cifar100_optimal_cells}, respectively.

\begin{figure}[htpb]
\centering

\rotatebox{90}{
\begin{minipage}{0.95\textheight}
    \centering
    \includegraphics[width=\linewidth]{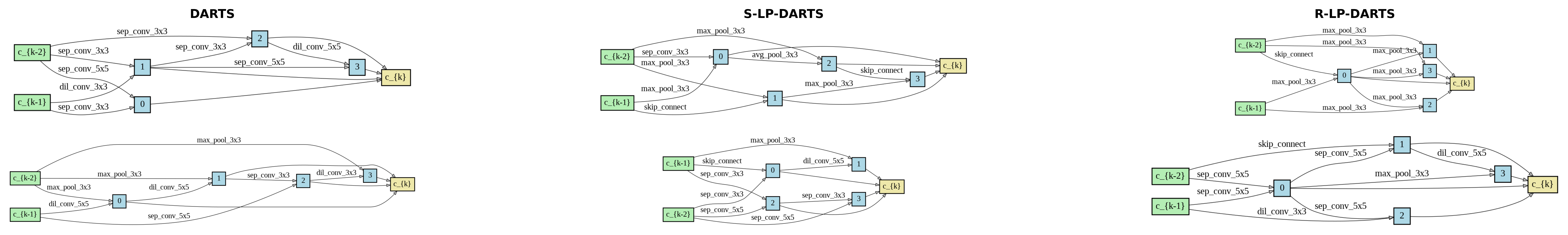}
    \captionof{figure}{Optimal normal and reduction cells searched on CIFAR-10 dataset.}
    \label{cifar10_optimal_cells}
\end{minipage}
}
\hspace{2.5cm}
\rotatebox{90}{
\begin{minipage}{0.95\textheight}
    \centering
    \includegraphics[width=\linewidth]{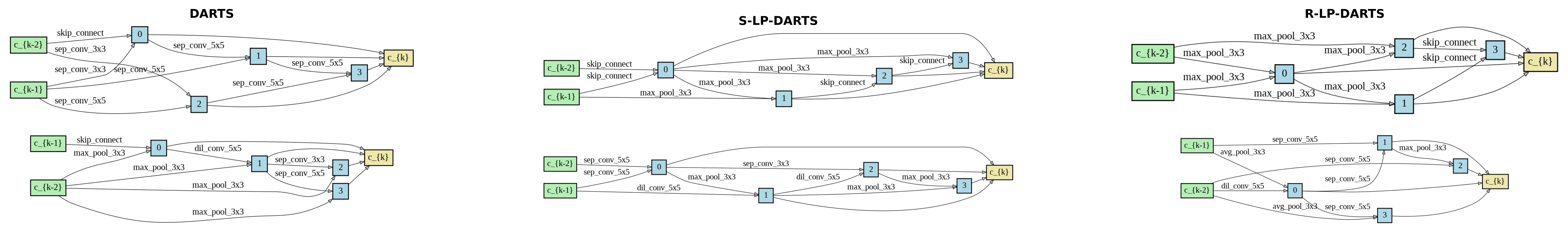}
    \captionof{figure}{Optimal normal and reduction cells searched on CIFAR-100 dataset.}
    \label{cifar100_optimal_cells}
\end{minipage}
}

\end{figure}

\section{Comparative Results} \label{app9}
NAS has evolved into one of the most dynamic subfields in deep learning, with a diverse set of methodologies demonstrating competitive results across benchmark datasets such as CIFAR-10, CIFAR-100, ImageNet, and more recently, large-scale transformer and graph domains. To provide a unified perspective on performance trends, this section consolidates results from key NAS algorithms, highlighting accuracy, parameter efficiency, and computational requirements. The data presented are drawn from original studies and comparative analysis in the NAS literature \cite{zoph2018learning, pham2018efficient, real2019regularized, DBLP:conf/iclr/LiuSY19, xu2019pcdarts, cai2018proxylessnas, wu2019fbnet, tan2019mnasnet, tan2019efficientnet, cai2019once, chen2021autoformer, gao2019graphnas, chen2019progressive, cai2024sto}. Table~\ref{tab:comparative_main} summarizes representative NAS methods spanning RL, EC, differentiable optimization, and meta-learning paradigms. Each approach is evaluated based on Top-1 accuracy, parameter count, and computational cost (architecture search time) measured in GPU days. The results clearly show the efficiency gains achieved by differentiable NAS and one-shot frameworks compared to earlier black-box methods.\\

\begin{table}[htb]
\centering
\caption{Comparative performance of representative NAS methods on CIFAR-10 and ImageNet benchmarks.}
\label{tab:comparative_main}
\resizebox{\textwidth}{!}{
\begin{tabular}{lcccccc}
\toprule
\textbf{Method} & \textbf{Search Type} & \textbf{Search Space} & \textbf{Dataset} & \textbf{Top-1 Acc. (\%)} & \textbf{Params (M)} & \textbf{GPU Days} \\
\midrule
NASNet-A \citep{zoph2018learning} & RL-based & Cell-based & CIFAR-10 (ImageNet) & 97.35 (74.0) & 3.3 (5.3) & 2000 \\
ENAS \citep{pham2018efficient} & RL (Weight Sharing) & Cell-based & CIFAR-10 & 97.11 & 4.6 & 0.5 \\
AmoebaNet-A \citep{real2019regularized} & Evolutionary & Cell-based & CIFAR-10 (ImageNet) & 96.66 (74.5) & 3.2 (5.1) & 3150 \\
DARTS \citep{DBLP:conf/iclr/LiuSY19} & Differentiable & Continuous & CIFAR-10 (ImageNet) & 97.24 (73.3) & 3.3 (4.7) & 4 \\
PC-DARTS \citep{xu2019pcdarts} & Differentiable & Continuous & CIFAR-10 (ImageNet) & 97.43 (74.9) & 3.6 (5.3) & 0.1 \\
ProxylessNAS \citep{cai2018proxylessnas} & Differentiable & Hardware-aware &  CIFAR-10 (ImageNet) & 97.92 (75.1) & 5.7 (7.1) & 4 \\
FBNet \citep{wu2019fbnet} & Differentiable & Hardware-aware & ImageNet & 74.9 & 5.5 & 9 \\
MnasNet \citep{tan2019mnasnet} & RL + Hardware & Mobile Search & ImageNet & 75.2 & 3.9 & 4.5 \\
EfficientNet-B0 \citep{tan2019efficientnet} & Compound Scaling & Platform-aware & ImageNet & 76.3 & 5.3 & - \\
Once-for-All \citep{cai2019once} & One-shot (Meta) & Elastic Supernet & ImageNet & 76.9 & 7.7 & 1.7 \\
Autoformer \citep{chen2021autoformer} & Differentiable & Transformer & ImageNet & 82.4 & 54.0 & - \\
GraphNAS \citep{gao2019graphnas} & RL-based & Graph Neural Nets & Cora (Citeseer) & 84.2 (73.1) & - & - \\
\bottomrule
\end{tabular}}
\end{table}

\end{document}